\def\year{2016}
\documentclass[letterpaper]{article}
\pdfoutput=1
\usepackage{aaai}
\usepackage{times}
\usepackage{helvet}
\usepackage{courier}
\usepackage{calc}
\usepackage{amssymb}
\usepackage{amstext}
\usepackage{amsmath}
\usepackage{amsthm}
\usepackage{natbib}
\usepackage[small]{caption}
\usepackage{multirow}
\usepackage{array}
\usepackage[ruled,linesnumbered,vlined]{algorithm2e}

\newcommand{\uni}{\cup} 

\newtheorem{definition}{Definition}
\newtheorem{lemma}{Lemma}
\newtheorem{theorem}{Theorem}
\newtheorem{proposition}{Proposition}
\newtheorem{example}{Example}

\begin{document}
%
\title{On Stochastic Belief Revision and Update and their Combination}
\author{Gavin Rens\\
Centre for Artificial Intelligence Research,\\ University of KwaZulu-Natal, School of Mathematics, Statistics and Computer Science and\\ CSIR Meraka, South Africa\\
Email: gavinrens@gmail.com
}
\maketitle
\begin{abstract}
\begin{quote}
I propose a framework for an agent to change its probabilistic beliefs when a new piece of propositional information $\alpha$ is observed. Traditionally, belief change occurs by either a revision process or by an update process, depending on whether the agent is \textit{informed} with $\alpha$ in a static world or, respectively, whether $\alpha$ is a `signal' from the environment due to an event occurring.
Boutilier suggested a unified model of qualitative belief change, which ``combines aspects of revision and update, providing a more realistic characterization of belief change.'' In this paper, I propose a unified model of \textit{quantitative} belief change, where an agent's beliefs are represented as a probability distribution over possible worlds. As does Boutilier, I take a dynamical systems perspective.
The proposed approach is evaluated against several rationality postulated, and some properties of the approach are worked out.
\end{quote}
\end{abstract}

\noindent
Information acquired can be due to evolution of the world or revelation about the world.
That is, one may notice via some `signal' generated by the changing environment that the environment has changed, or, one may be informed by an independent agent in a static environment that some `fact' holds.

In the present work, I deal with belief change of agents who handle uncertainty by maintaining a probability distribution over possible situations.
The agents in this framework also have models for nondeterministic events, and noisy observations. Noisy observation models can model imperfect sensory equipment for receiving environmental signals, but they can also model untrustworthy informants in a static world.

In this paper, I provide the work of \citet{b98a} as background, because it has several connections with and was the seed for the present work. However, I do not intend simply to give a probabilistic version of his Generalized Update Semantics.
Whereas \citet{b98a} presents a model for unifying qualitative belief revision and update, I build on his work to present a unified model of belief revision and update in a stochastic (probabilistic) setting. I also take a dynamical systems perspective, like him. Due to my quantitative approach, an agent can maintain a probability distribution over the worlds it believes possible, using an \textit{expectation} semantics of change. This is in contrast to Boutilier's ``generalized update'' approach, which takes a most-plausible event and most-plausible world approach. Finally, my proposal requires a trade-off factor to mix the changes in probability distribution over possible worlds brought about due to the probabilistic belief revision process and, respectively, the probabilistic belief update process. Boutilier's model has revision and update more tightly coupled. For this reason, his approach is better called ``unified'' while mine is called ``hybrid''.

The belief change community does not study probabilistic belief \textit{update}; it is studied almost exclusively in frameworks employing Bayesian conditioning -- for modeling events and actions in dynamical domains (e.g., DBNs, MDPs, POMDPs) \citep[e.g.]{kf09,pm10}.
The part of my approach responsible for updating stays within the Bayesian framework, but combines the essential elements of belief update with \textit{unobservable events} and belief update as partially observable Markov decision process (POMDP) \textit{state estimation}.

On the other hand, there is plenty of literature on probabilistic belief \textit{revision} \citep[e.g.]{v99a,gh98,k08,yl08}. The subject is both deep and broad. There is no one accepted approach and to argue which is the best is not the focus of this paper. I shall choose one reasonable method for probabilistic belief revision suitable to the task at hand. 


In the first section, Boutilier's `generalized update' is reviewed. Then, in the next section, I introduce stochastic update and stochastic revision, culminating in the `hybrid stochastic belief change' (HSBC) approach. The final section presents an example inspired by Boutilier's article \citeyearpar{b98a} and analyses the results.

Some proofs of propositions are omitted to save space; they are available on request.

\section{Boutilier's Generalized Update}\label{sec:bgu}

I use Boutilier's notation and descriptions, except that I am more comfortable with $\alpha$ and $\beta$ to represent sentences, instead of $A$ and $B$.
It is assumed that an agent has a deductively closed belief set $K$, a set of sentences drawn from some logical language reflecting the agent's beliefs about the current state of the world.
For ease of presentation, I assume a logically finite, classical propositional language, denoted $L$ ($L_\mathit{CPL}$ in \citet{b98a}), and consequence operation $\mathit{Cn}$.
The belief set $K$ will often be generated by some finite knowledge base $\mathit{KB}$ (i.e., $K=\mathit{Cn}(\mathit{KB})$).
The identically true and false propositions are denoted $\top$ and $\bot$, respectively.
Given a set of possible worlds $W$ (or valuations over $L$) and $\alpha\in L$, the set of $\alpha$-worlds, that is, the elements of $W$ satisfying $\alpha$, is denoted by $||\alpha||$.
The worlds satisfying all sentences in a set $K$ is denoted $||K||$.

\subsection{Update}

Given a belief set $K$, an agent will often observe a change in the world $\alpha$, requiring the agent to change $K$. This is the \textit{update} of $K$ by $\alpha$, denoted $K^\diamond_\alpha$.

``$||\mathit{KB}||$ represents the set of possibilities we are prepared to accept as the actual state of affairs. Since observation $\alpha$ is the result of some change in the actual world, we ought to consider, for each possibility $w\in ||\mathit{KB}||$, the most plausible way (or ways) in which $w$ might have changed in order to make $\alpha$ true. That is, we want to consider the most plausible evolution of world $w$ into a world satisfying the observation $\alpha$. To capture this intuition, \citet{km91} propose a family of preorders $\{\leq_w\mid w\in W\}$, where each $\leq_w$ is a reflexive, transitive relation over $W$. We interpret each such relation as follows: if $u\leq_w v$ then $u$ is at least as plausible a change relative to $w$ as is $v$; that is, situation $w$ would more readily evolve into $u$ than it would into $v$.

Finally, a faithfulness condition is imposed: for every world $w$, the preorder $\leq_w$ has $w$ as a minimum element; that is, $w<_w v$ for all $v\neq w$. Naturally, the most plausible candidate changes in $w$ that result in $\alpha$ are those worlds $v$ satisfying $\alpha$ that are minimal in the relation $\leq_w$. The set of such minimal $\alpha$-worlds for each relation $\leq_w$, and each $w\in ||\mathit{KB}||$, intuitively capture the situations we ought to accept as possible when updating $\mathit{KB}$ with $\alpha$,'' \citep[p.~9]{b98a}. In other words,
\[
||\mathit{KB} \diamond\alpha|| = \bigcup_{w\in||\mathit{KB}||}\{\mathit{Min}(\alpha,\leq_w)\},
\]
where $\mathit{Min}(\alpha,\leq_w)$ specifies the minimal $\alpha$-worlds with respect to the preorder $\leq_w$.
Then
$
K^\diamond_\alpha = \mathit{Cn}(\mathit{KB} \diamond\alpha)
$,
where $K$ is the belief set associated with $\mathit{KB}$.

\subsection{Revision}

Given a belief set $K$, an agent will often obtain information $\alpha$ in a static world, which must be incorporated into $K$. This is the \textit{revision} of $K$ by $\alpha$, denoted $K^*_\alpha$.

The AGM theory of belief revision \citep{agm85} provides a set of guidelines, in the form of the postulates, governing the process. ``Unfortunately, while the postulates constrain possible revisions, they do not dictate the precise beliefs that should be retracted when $\alpha$ is observed. An alternative model of revision, based on the notion of epistemic entrenchment \citep{g88}, has a more constructive nature,'' \citep[p.~6]{b98a}.

``Semantically, an entrenchment relation (hence a revision function) can be modeled using an ordering on possible worlds reflecting their relative plausibility \citep{g88b,b94}.
However, rather than use a qualitative ranking relation, we adopt the presentation of \citep{s88,gp92} and rank all possible worlds using a $\kappa$-ranking. Such a ranking $\kappa: W \to N$ assigns to each world a natural number reflecting its plausibility or degree of believability.
If $\kappa(w)< \kappa(v)$ then $w$ is more plausible than $v$ or more consistent with the agent's beliefs.
We insist that $\kappa^{-1}(0) \neq \emptyset$, so that maximally plausible worlds are assigned rank 0.
These maximally plausible worlds are exactly those consistent with the agent's beliefs; that is, the epistemically possible worlds according to $K$ are those deemed most plausible in $\kappa$ (see \citet{s88} for further details).
We sometimes assume $\kappa$ is a partial function, and loosely write $\kappa(w)=\infty$ to mean $\kappa(w)$ is not defined (i.e., $w$ is not in the domain of $\kappa$, or $w$ is impossible),'' \citep[p.~6]{b98a}.

A $\kappa$-ranking captures the entrenchment of the agent's beliefs in its belief set $K$. This entrenchment determines how $K$ will be revised when the agent receives new information / makes an observation $\alpha$. $\kappa$ induces a belief set as follows.
\[
K=\{\alpha\in L\mid \kappa^{-1}(0)\subseteq ||\alpha||\}.
\]
Due to the ranking or entrenchment of knowledge provided by $\kappa$, $\kappa$ is considered an \textit{epistemic state}.

``In other words, the set of most plausible worlds (those such that $\kappa(w)=0$) determine the agent's beliefs.
The ranking $\kappa$ also induces a revision function: to revise by $\alpha$ an agent adopts the most plausible $\alpha$-worlds as epistemically possible,'' \citep[p.~6]{b98a}.

Let $W_i = \{w\in W\mid \kappa(w)=i\}$.
And let $\mathit{Min}(\alpha,\kappa)$
be the set $W_i$ with the least $i$ such that for all $w^i\in W_i$, $w_i\models\alpha$.
Then
\[
K^*_\alpha := \{\beta \in L\mid \mathit{Min}(\alpha,\kappa)\subseteq ||\beta||\}.
\]
In words, the belief set revised by $\alpha$ contains all those sentences entailed by the set of worlds with the same rank, where that rank is the least such that they are all $\alpha$-worlds.

\subsection{Generalized Update}

As explained in the introduction, my intention with this paper is not to give a probabilistic version of the Generalized Update approach \citep{b98a}. For completeness, however, I sketch the approach here covering the approach in detail would take up unnecessary space without lending much insight into my Hybrid Stochastic Belief Change approach.

Boutilier motivates the need for a generalized update method which includes revision, by claiming that KM update \citep{km91} is insufficient. He provides the following example adopted from \citet{m90}. Suppose you want to test whether the contents of a beaker are chemically acid or base. If it is acid, a piece of litmus paper will turn red, if base, the paper will turn blue. Suppose that the test has not yet been performed, but you believe that the contents in the beaker are acidic. When the litmus paper is dipped into and pulled out of the beaker, the paper turns blue, indicating a basic compound. ``Unfortunately, the KM theory does not allow this to take place. [...] One is forced to accept that, if the contents were acidic (in which case it should turn red), some extraordinary change occurred (the test failed, the contents of the beaker were switched, etc.). [...] Of course, the right thing to do is simply admit
that the beaker did not, in fact, contain an acid---the agent should \textit{revise} its beliefs about the contents of the beaker,'' \citep[p.~13]{b98a}.

Boutilier adopts an event-based approach where a set of events $E$ is assumed. These events are allowed to be nondeterministic, and each possible outcome of an event is ranked according to its plausibility via a $\kappa$-ranking.
``As in the original event-based semantics, we will assume each world has an event ordering associated with it that describes the plausibility of various event occurrences at that world,'' \citep[p.~14]{b98a}.

A \emph{generalized update model} is then defined as $\langle W,\kappa,E,\mu\rangle$, where $W$ is a set of worlds (the agent's epistemic state), $\kappa$ is a ranking over $W$, $E$ is a set of events (mappings over $W$), and $\mu$ is an event ordering (a set of mappings over $E$).

As with KM update, updates usually occur in response to some observation, with the assumption that something occurred to cause this observation. After observing $\alpha$, an agent should adjust its beliefs by considering that only the most plausible transitions leading to $\alpha$ actually occurred. The set of possible $\alpha$-transitions are those transitions leading to $\alpha$-worlds.
The most plausible $\alpha$-transitions are those possible $\alpha$-transitions with the minimal $\kappa$-ranking. Given that $\alpha$ has actually been observed, an agent should assume that one of these
transitions describes the actual course of events. The worlds judged to be epistemically possible are those that result from the most plausible of these transitions.

\citet{b98a} has a proposition that states that generalized belief update as described above is equivalent to ``first determining the predicted updated ranking $\kappa^\diamond$ followed by standard (AGM) revision by $\alpha$ with respect to $\kappa^\diamond$,'' \citep[p.~16]{b98a}.
$\kappa^\diamond$ is determined by taking the worlds in the current possible worlds $||K||$ (induced from $\kappa$) and shifting them to all possible worlds given all possible transitions given all possible events (the actual event is unknown), taking into account the relevant plausibility rankings.

\section{Stochastic Belief Change}\label{sec:hsc}

I now consider agents who deal with uncertainty by maintaining a probability distribution over possible situations (worlds) they could be in.
Let a \textit{belief state} $b$ be defined as the set $\{(w,p)\mid w\in W, p\in[0,1]\}$, where $\sum_{(w,p)\in b}p=1$. The probability of being in $w$ is denoted $b(w)$. That is, $b$ is a probability distribution over all the worlds in $W$. In the hybrid stochastic belief change (HSBC) framework, an agent maintains a belief state, which changes as new information is received or observed.

An agent is assumed to have a model of how the world works.
\begin{definition}
\label{def:sbc-model}
The \emph{stochastic belief change model} $M$ has the form $\langle W,\varepsilon,T,E,O,\mathit{os}\rangle$, where
\begin{itemize}
\itemsep=0pt
\item
$W$ is a set of possible worlds, 
\item
$\varepsilon$ is a set of events, 
\item
$T:(W\times \varepsilon\times W)\to[0,1]$ is a \emph{transition function} such that for every $e\in \varepsilon$ and $w\in W$, $\sum_{w'\in W}T(w,e,w')=1$
($T(w,e,w')$ models the probability of a transition to world $w'$, given the occurrence of event $e$ in world $w$),
\item
$E$ is the event likelihood function ($E(e,w)=P(e\mid w)$, the probability of the occurrence of event $e$ in $w$),
\item
$O:(L\times W)\to[0,1]$ is an \emph{observation function} such that for every world $w$, $\sum_{\alpha\in \Omega}O(\alpha,w)=1$ ($O(\alpha, w)$ models the probability of observing $\alpha$ in $w$), where $\Omega\subset L$ is the set of possible observations, up to equivalence, and where if $\alpha\equiv\beta$, then $O(\alpha, w)=O(\beta, w)$, for all worlds $w$.\footnote{$\equiv$ denotes logical equivalence.}
\item
$\mathit{os}:(\Omega\times W)\to[0,1]$
($\mathit{os}(\alpha,w)$ is the agent's ontic strength for $\alpha$ perceived in $w$.)
\end{itemize}
\end{definition}

\begin{definition}
$b(\alpha):=\sum_{w\in W,w\models\alpha}b(w)$.
\end{definition}

Let $b^\circ_\alpha:=b\circ\alpha$ so that we can write $b^\circ_\alpha(w)$, where $\circ$ is any update or revision operator.

Often, in the exposition of this paper, a world will be referred to by its truth vector. For instance, if the vocabulary is $\{q,r\}$ and $w_3\models \lnot q\land r$, then $w_3$ may be referred to as $01$.

For parsimony, let $b=\langle p_1,\ldots, p_n \rangle$ be the probabilities that belief state $b$ assigns to $w_1,\ldots, w_n$ where $\langle w_1,w_2,w_3,w_4\rangle$ $=$ $\langle 11,10,01,00\rangle$, and $\langle w_1,w_2,\ldots,w_8\rangle$ $=$ $\langle 111,110,\ldots,000\rangle$.

\subsection{Update}

Transitions associated with the observation of $\alpha$ from a world $w$ in the current belief state $b_\mathit{cur}$ to a world $w'$ could be caused by different events.
According to \citet{b98a}, update can be defined as
\begin{align*}
&b^\mathit{event}_\mathit{new} := \Big\{(w',p')\mid w'\in W, p'=\\
&\qquad\sum_{w\in W}\sum_{e\in\varepsilon}T(w,e,w')E(e, w)b_\mathit{cur}(w)\Big\}.
\end{align*}
Because the actual event is unobservable/hidden, $p'$ is the \textit{expected} probability of reaching $w'$, given the event probabilities.

In partially observable Markov decision process (POMDP) theory \citep{a65,m82,l91}, events are actions chosen by the agent (and thus observable) and observations are hidden. Then, given current belief state $b_\mathit{cur}$, selected action $a$ and observation $o$, the \textit{state estimation function} is defined by 
\begin{align*}
&b^\mathit{pomdp}_\mathit{new} := \Bigg\{(w',p')\mid w'\in W, p'=\\
&\qquad\frac{O(o,a,w')\sum_{w\in W}T(w,a,w')b_\mathit{cur}(w)}{P(o\mid a,b_\mathit{cur})}\Bigg\},
\end{align*}
where 
$\Omega$ is a set of observation objects and $O:(\Omega\times A\times W)\to[0,1]$ is an \textit{observation function}, such that for every $a$ and $w'$, $\sum_{o\in \Omega}O(o,a,w')=1$.
$O(o,a,w')$ models the probability of perceiving $o$ in arrival world $w'$, given the execution of some action $a\in A$.
Note that $P(o\mid a,b_\mathit{cur})$ is a normalizing constant.

But what is the probabilistic update, given new information/evidence $\alpha$?
I suggest that $\alpha$ is the (overt) `signal' generated by the (covert) event.
An important question is, When is $\alpha$ received -- in the current/departure world ($w_c$) or in the new/arrival world ($w_n$)? Although it is not clear to me, in POMDP theory, observations are always assumed to be received in the arrival world -- I shall assume the same.

In the present framework, actions are not selected by the agent, but by nature. In other words, actions are considered to be events occurring in the environment, uncontrollable by the agent. Further, at the present stage of research, I shall assume that the agent has a less detailed observation model, that is, an agent only knows $O(\alpha, w_n)$, the probability of perceiving $\alpha$ in arrival world $w_n$ (defined in Def.~\ref{def:sbc-model}).
Hence, I propose to weight $b^\mathit{event}_\mathit{new}(w')$ by $O(\alpha,w_n)$ when receiving new information $\alpha$ and one knows that one's belief state should be \textit{updated} (due to an evolving world).
Then we can define
\begin{definition}
\begin{align*}
b\diamond\alpha :=& \big\{(w',p')\mid w'\in W, p'=\\
&\frac{1}{\gamma}O(\alpha,w')\sum_{w\in W}\sum_{e\in\varepsilon}T(w,e,w')E(e, w)b(w)\big\},
\end{align*}
where $\gamma$ is a normalizing factor.
\end{definition}

As far as I know, no-one has proposed rationality postulates for probabilistic update. The reason is likely due to probabilistic update being defined in terms of standard probability theory. The axioms of probability theory have been argued to be rational for several decades (although it is not without its detractors).

The following basic postulates for my probabilistic belief update are proposed. (Unless stated otherwise, it is assumed that $\alpha$ is logically satisfiable, i.e., $\vdash\lnot\alpha$ is false.)
\begin{itemize}
\itemsep=0pt
\item[]($P^\diamond 1$) $b^\diamond_\alpha$ is a belief state iff not $\vdash\lnot\alpha$
\item[]($P^\diamond 2$) $b^\diamond_\alpha(\alpha)=1$
\item[]($P^\diamond 3$) If $\alpha\equiv\beta$, then $b^\diamond_\alpha=b^\diamond_\beta$
\end{itemize}

\begin{proposition}
If $b^\diamond_\alpha(\alpha)>0$, it is not necessary that $b^\diamond_\alpha(\alpha)=1$.
\label{prp:not-necess-accepted}
\end{proposition}
\begin{proof}
Let the vocabulary be $\{q,r\}$. Let $b=\langle0.4,0,0.1,0.5\rangle$. Let there be only one event $e$. Let the transition function be specified as $T(11,e,11)=0.5$, $T(11,e,10)=0.5$, $T(10,e,01)=1$, $T(01,e,00)=1$, $T(00,e,11)=1$.
Let $E(e, w)=1$ for all $w\in W$.
Let the evidence be $q$.
Let $O(q, 11)=0.2$, $O(q, 10)=0$, $O(q, 01)=0$, $O(q, 00)=0.3$.
Then applying operation $\diamond$ to $b$ produces $b^\diamond_q =\langle0.82,0,0,0.18\rangle$.
Hence, $b^\diamond_q(q)=0.82\neq1$.
\end{proof}

\medskip
Although the following proposition is mostly negative, the reader will soon see that constraining the stochastic belief change model
to be `rational', the negative postulates become positive.
\begin{proposition}
Postulate ($P^\diamond 3$) holds, while ($P^\diamond 1$) and ($P^\diamond 2$) do not hold.
\label{prp:update-postulates}
\end{proposition}

\begin{definition}
\label{def:event-rational}
We say event $e$ is \emph{event-rational} when for all $w\in W$: there exists a $w'$ such that $T(w,e,w')>0$ iff $E(e, w)>0$.
\end{definition}

\begin{definition}
\label{def:e-signal}
We say $\alpha$ is an $e$\emph{-signal} when for all $w'\in W$: there exists a $w$ such that $T(w,e,w')>0$ iff $O(\alpha,w')>0$.
\end{definition}

\begin{definition}
\label{def:observation-rational}
We say a model $M$ is \emph{observation-rational} iff for all $\alpha$, whenever $\vdash\lnot\alpha$, $O(\alpha, w)=0$ for all $w\in W$.
\end{definition}

The proposition below says that if one is rational w.r.t. observations and w.r.t. some event, and $\alpha$ is a signal produced by that event, then updating on $\alpha$ is defined.
\begin{proposition}
If $M$ is observation-rational, there exists an event $e\in\varepsilon$ which is event-rational and $\alpha$ is an $e$-signal, then $b^\diamond_\alpha$ is a belief state iff not $\vdash\lnot\alpha$ (i.e., then ($P^\diamond 1$) holds).
\label{prp:update-defined}
\end{proposition}
%
%
%
%

($P^\diamond 2$) does not hold under the antecedents of Proposition~\ref{prp:update-defined}. Another definition is required as qualification:
\begin{definition}
\label{def:trustworthy}
We say evidence $\alpha$ is \emph{trustworthy} iff for all $w\in W$, if $w\not\models\alpha$, then $O(\alpha, w)=0$.
\end{definition}

The proposition below says that if $\alpha$ is trustworthy, one is rational w.r.t. some event, and $\alpha$ is a signal produced by that event, then one should accept $\alpha$ in the updated belief state.
\begin{proposition}
If $\alpha$ is trustworthy, there exists an event $e\in\varepsilon$ which is event-rational and $\alpha$ is an $e$-signal, then $b^\diamond_\alpha(\alpha)=1$ (i.e., then ($P^\diamond 2$) holds).
\label{prp:update-postulate2}
\end{proposition}
\begin{proof}
Not $\vdash\lnot\alpha$ is assumed by default.
Recall that $b^\diamond_\alpha(\alpha)=\sum_{w\in W,w\models\alpha}b^\diamond_\alpha(w)$.
Referring to the ($\Leftarrow$) part of the proof of Proposition~\ref{prp:update-defined}, $b^\diamond_\alpha(\alpha)$ is a belief state and thus $\sum_{w\in W}b^\diamond_\alpha(w)=1$.
Hence, for $b^\diamond_\alpha(\alpha)$ to be less than 1, there must exist a $w'\in W$ s.t. $w'\not\models\alpha$ and $b^\diamond_\alpha(w')>0$. But then $O(\alpha, w')>0$. Therefore, for ($P^\diamond 2$) not to hold, an agent needs to believe that $O(\alpha, w')>0$ for some world $w'$ where $w'\not\models\alpha$. But then $\alpha$ cannot be trustworthy (i.e., then ($P^\diamond 2$) holds.
\end{proof}

\begin{definition}[G\"ardenfors, 1988]
A probabilistic belief change operation $\circ$ is said to be \emph{preservative} iff for all belief states $P$ and for all propositions $\alpha$ and $\beta$, if $P(\alpha)>0$ and $P(\beta)=1$, then $P^\circ_\alpha(\beta)=1$.
\label{def:preservative}
\end{definition}

\begin{proposition}
Operation $\diamond$ is not preservative.
\label{prp:update-not-preservative}
\end{proposition}

\begin{definition}
We say evidence $\alpha$ is \emph{$\beta$-trustworthy} if for all $w\in W$, if $w\not\models\beta$, then $O(\alpha, w)=0$.
\end{definition}

\begin{proposition}
If $b^\diamond_\alpha(\beta)$ is a belief state, $b(\beta)=1$ and $\alpha$ is $\beta$-trustworthy, then $b^\diamond_\alpha(\beta)=1$.
\label{prp:accepting-dueto-simulation}
\end{proposition}
\begin{proof}
$\sum_{w\in W}b^\diamond_\alpha(w)=1$.
Hence, for $b^\diamond_\alpha(\beta)$ to be $<1$, there must exist a $w^\times\in W$ s.t. $w^\times\not\models\beta$ and $b^\diamond_\alpha(w^\times)>0$.
And because $b(\beta)=1$, $b(w^\times)=0$. 
So some probability mass must have been shifted from some $\beta$-world to the non-$\beta$-world $w^\times$.

By definition, $b^\diamond_\alpha(w^\times)$ $=$ $\frac{1}{\gamma}$ $O(\alpha,w^\times)$ $\sum_{w\in W}\sum_{e\in\varepsilon}T(w,e,w^\times)E(e, w)b(w)$. So for $b^\diamond_\alpha(w^\times)$ to be $>0$, $O(\alpha,w^\times)$ must be $>0$.

However, because $\alpha$ is $\beta$-trustworthy, $O(\alpha, w^\times)=0$.
Hence, $O(\alpha,w^\times)\not>0$ and $b^\diamond_\alpha(\beta)\not<1$.
\end{proof}

\begin{proposition}
$b^\diamond_{\alpha\land\beta}\neq(b^\diamond_\alpha)^\diamond_\beta$.
\end{proposition}
\begin{proof}
For instance, consider the example used in the proof of Proposition~\ref{prp:not-necess-accepted}. Let $\alpha$ be $q$ and let $\beta$ be $q\land r$. Note that $\alpha\land\beta$ is then logically equivalent to $q\land r$.
Let $O(q\land r,11)=O(q\land r,10)=0.5$ and $O(q\land r,01)=O(q\land r,00)=0$.

We know that $b^\diamond_q=\langle0.82,0,0,0.18\rangle$.
Then $(b^\diamond_q)^\diamond_{q\land r} = \langle1,0,0,0\rangle$.
On the other hand, $b^\diamond_{q\land r} = \langle0.875,0,0,0.125\rangle$.
\end{proof}

\subsection{Revision}

Using Bayes' Rule\footnote{Bayes' Rule states (in the notation of this paper) that $P(w\mid \alpha) = P(\alpha\mid w)P(w)/P(\alpha)$ or $P(w\mid \alpha) = P(\alpha\mid w)P(w)/\sum_{w'\in W}P(\alpha\mid w')P(w')$.}
, $P(w_n\mid \alpha)$ can be determined:
\[
P(w\mid \alpha):=\frac{O(\alpha, w)b(w)}{\sum_{w'\in W}O(\alpha, w')b(w')}.\]
Note that if $O(\alpha, w)=0$, then $P(w\mid\alpha)=0$.

It is not yet universally agreed what revision means in a probabilistic setting. In classical belief change, it is understood that if the new information $\alpha$ is consistent with the agent's current beliefs $\mathit{KB}$, then revision is equivalent to belief expansion (denoted $+$), where expansion is the logical consequences of $\mathit{KB}\uni\{\alpha\}$.
It is mostly agreed upon
that Bayesian conditioning corresponds to classical belief expansion. This is evidenced by Bayesian conditioning ($\mathsf{BC}$) being defined only when $b(\alpha)\neq0$.
In other words,
one could define revision to be 
\[
b\:\mathsf{BC}\:\alpha :=\{(w,p)\mid w\in W, p= P( w\mid\alpha)\},	
\]
\textit{as long as} $P(\alpha) \neq 0$.\footnote{Note that in my notation, $b(\alpha)$ is equivalent to $P(\alpha)$.}

To accommodate cases where $b(\alpha)\neq0$, that is, where $\alpha$ contradicts the agent's current beliefs and its beliefs need to be revised in the stronger sense, we shall make use of \textit{imaging}.
Imaging was introduced by \citet{l76} as a means of revising a probability function. It has also been discussed in the work of, for instance, \citet{g88,dp93,cnss14,rm15b}.
The following version of imaging must not be regarded as a fundamental part of the larger belief change framework presented here; it should be regarded as a place-holder or suggestion for the `revision-module' of the framework.
Informally, Lewis's original solution for accommodating contradicting evidence $\alpha$ is to move the probability of each world to its closest, $\alpha$-world. Lewis made the strong assumption that every world has a \textit{unique} closest $\alpha$-world. More general versions of imaging allow worlds to have \textit{several}, equally proximate, closest worlds.

\citet{g88} calls one of his generalizations of Lewis's imaging \textit{general imaging}. Our method is also a generalization.
We thus refer to his as \textit{G\"ardenfors's general imaging} and to our method as \textit{generalized imaging} to distinguish them. It should be noted that these imaging methods are general revision methods and can be used in place of Bayesian conditioning for expansion. ``Thus imaging is a more general method of describing belief changes than conditionalization,'' \citep[p.~112]{g88}.

Let $\mathit{Min}(\alpha,w,d)$ be the set of $\alpha$-worlds closest to $w$ measured with $d$. Formally,
\begin{align*}
&\mathit{Min}(\alpha,w,d):=\\
&\qquad\{w'\in||\alpha||\mid \forall w''\in||\alpha||, d(w',w)\leq d(w'',w)\},
\end{align*}
where $d(\cdot)$ is some acceptable measure of distance between worlds (e.g., Hamming or Dalal distance). It must also obey the faithfulness condition that for every world $w$, $d(w,w)<d(v,w)$ for all $v\neq w$.
\begin{example}
\label{ex:1}
Let the vocabulary be $\{q,r,s\}$.
Let $\alpha$ be $(q\land r)\lor(q\land\lnot r\land s)$. Suppose $d$ is Hamming distance. Then
\begin{align*}
&\mathit{Min}((q\land r)\lor(q\land\lnot r\land s),111,d)=\{111\}\\
&\mathit{Min}((q\land r)\lor(q\land\lnot r\land s),110,d)=\{110\}\\
&\mathit{Min}((q\land r)\lor(q\land\lnot r\land s),101,d)=\{101\}\\
&\mathit{Min}((q\land r)\lor(q\land\lnot r\land s),100,d)=\{110,101\}\\
&\mathit{Min}((q\land r)\lor(q\land\lnot r\land s),011,d)=\{111\}\\
&\mathit{Min}((q\land r)\lor(q\land\lnot r\land s),010,d)=\{110\}\\
&\mathit{Min}((q\land r)\lor(q\land\lnot r\land s),001,d)=\{101\}\\
&\mathit{Min}((q\land r)\lor(q\land\lnot r\land s),000,d)=\{110,101\}
\end{align*}
\end{example}

Then generalized imaging (denoted $\mathsf{GI}$) is defined as
\begin{definition}
\begin{align*}
&b\:\mathsf{GI}\:\alpha :=\big\{(w,p)\mid w\in W, p=0 \mbox{ if } w\not\in||\alpha||,\\
& \qquad \mbox{else }p=\sum_{\substack{w'\in W\\w\in\mathit{Min}(\alpha,w',d)}}b(w')/|\mathit{Min}(\alpha,w',d)|\big\}.
\end{align*}
\end{definition}
\begin{example}
\label{ex:2}
Continuing on Example~\ref{ex:1}:
Let $b=\langle0,0.1,0,0.2,0,0.3,0,0.4\rangle$.

$(q\land r)\lor(q\land\lnot r\land s)$ is abbreviated as $\alpha$.

\smallskip
$b^\mathsf{GI}_\alpha(111) = \sum_{\substack{w'\in W\\111\in\mathit{Min}(\alpha,w',d)}}b(w')/|\mathit{Min}(\alpha,w',d)|$ $=$ $b(111)/|\mathit{Min}(\alpha,111,d)| + b(011)/|\mathit{Min}(\alpha,011,d)|$ $=$ $0/1 + 0/1$ $=0$.

\smallskip
$b^\mathsf{GI}_\alpha(110) = \sum_{\substack{w'\in W\\110\in\mathit{Min}(\alpha,w',d)}}b(w')/|\mathit{Min}(\alpha,w',d)|$ $=$ $b(110)/|\mathit{Min}(\alpha,110,d)| + b(100)/|\mathit{Min}(\alpha,100,d)| + b(010)/|\mathit{Min}(\alpha,010,d)| + b(000)/|\mathit{Min}(\alpha,000,d)|$ $=$ $0.1/1 + 0.2/2 + 0.3/1 + 0.4/2$ $=0.7$.

\smallskip
$b^\mathsf{GI}_\alpha(101) = \sum_{\substack{w'\in W\\101\in\mathit{Min}(\alpha,w',d)}}b(w')/|\mathit{Min}(\alpha,w',d)|$ $=$ $b(101)/|\mathit{Min}(\alpha,101,d)| + b(100)/|\mathit{Min}(\alpha,100,d)| + b(001)/|\mathit{Min}(\alpha,001,d)| + b(000)/|\mathit{Min}(\alpha,000,d)|$ $=$ $0/1 + 0.2/2 + 0/1 + 0.4/2$ $=0.3$.

\smallskip
And $b^\mathsf{GI}_\alpha(100) = b^\mathsf{GI}_\alpha(011) = b^\mathsf{GI}_\alpha(010) = b^\mathsf{GI}_\alpha(001) = b^\mathsf{GI}_\alpha(000) = 0$.
\end{example}
Notice how the probability mass of non-$\alpha$-worlds is shifted to their closest $\alpha$-worlds. If a non-$\alpha$-world $w^\times$ with probability $p$ has $n$ closest $\alpha$-worlds (equally distant), then each of these closest $\alpha$-worlds gets $p/n$ mass from $w^\times$.

Recall that in the proposed framework, agents have access to an observation model (formalized via an observation function $O(\cdot,\cdot)$). Given enough computational power and time, it would be irrational for an agent to ignore its observation model when revising its beliefs. Another proposed definition for a stochastic belief revision operation based on imaging (denoted $\mathsf{OGI}$) is thus
\begin{definition}
\begin{align*}
&b\:\mathsf{OGI}\:\alpha :=\Big\{(w,p)\mid w\in W, p=\frac{O(\alpha,w)b^\mathsf{GI}_\alpha(w)}{\sum_{w'\in W}O(\alpha,w')b^\mathsf{GI}_\alpha(w')}\Big\},
\end{align*}
where the denominator is a normalizing factor.
\end{definition}

$b\:\mathsf{OGI}\:\alpha$ is \textit{not} defined as
\begin{align*}
& \Big\{(w,p)\mid w\in W, p=0 \mbox{ if } w\not\in||\alpha||,\\
& \quad \mbox{else }p=\sum_{\substack{w'\in W\\w\in\mathit{Min}(\alpha,w',d)}}O(\alpha,w')b(w')/|\mathit{Min}(\alpha,w',d)|\Big\},
\end{align*}
because $\alpha$ is assumed perceived in the new world $w$, not the old world $w'$.

Note that if $P(\cdot\mid\alpha)$ were used instead of $O(\alpha,\cdot)$, then $\mathsf{OGI}$ would be undefined whenever $b(\alpha)=0$. But this is exactly the problem we want to avoid by using imaging.
Another justification to rather use $O(\alpha,w)$ is that its value is positively correlated with $P(w\mid\alpha)$: If $O(\alpha,w)=0$, then $P(w\mid\alpha)=0$. If $P(w\mid\alpha)=1$, then $O(\alpha,w)$ is maximal in $b$ in the following sense: for all $w'\in W$, if $w'\neq w$, then either $b(w')=0$ or $O(\alpha,w')=0$, whereas $b(w)>0$ and $O(\alpha,w)>0$.

Note that the denominator my be zero, making $\mathsf{OGI}$ undefined in that case.
I shall deal with this issue a little bit later.

\begin{example}
\label{ex:3}
Recall from Example~\ref{ex:2} that $b^\mathsf{GI}_\alpha = \langle0,0.7,0.3,0,0,0,0,0\rangle$ and $\alpha$ is $(q\land r)\lor(q\land\lnot r\land s)$.
Let $O(\alpha,w)=0.3$, say, for all $w\in W$. Then $b^\mathsf{OGI}_\alpha =b^\mathsf{GI}_\alpha$.
Obviously, if the observation model carries no information with respect to $\alpha$, then it has no influence on the agent's revised beliefs.

Now let $O(\alpha,w)=0.3$ if $w\models r$, else $O(\alpha,w)=0.2$.
Then $b^\mathsf{OGI}_\alpha = \langle0.3\times0/0.23,0.2\times0.7/0.23,0.3\times0.3/0.23,
0.2\times0/0.23,0.3\times0/0.23,0.2\times0/0.23,
0.3\times0/0.23,0.2\times0/0.23\rangle =  \langle0,0.61,0.39,0,0,0,0,0\rangle$.
If the agent has an observation model telling it that $\alpha$ is more likely to be perceived in $r$-worlds than in $\lnot r$-worlds, then when it receives $\alpha$, the agent should be biased to believing that it is actually in an $r$-world.
However, the agent was certain that it was in a $\lnot r$-world when its belief state was $b$. $\mathsf{GI}$ thus pushes the agent to favour the $\alpha$-worlds being $\lnot r$-worlds. Hence, in this example there is tension between being in a $\lnot r$-world (due to previous beliefs) and being in an $r$-world (due to the observation model).
\end{example}

%
%
%

\begin{definition}
\begin{equation*}
b\:\mathsf{BCI}\:\alpha := \left\{
\begin{array}{rl}
b\:\mathsf{BC}\:\alpha & \text{if } b(\alpha) > 0\\
b\:\mathsf{OGI}\:\alpha & \text{if } b(\alpha) = 0
\end{array} \right.
\end{equation*}
\end{definition}

I denote the \emph{expansion} of belief state $b$ on $\alpha$ as $b^+_\alpha$ (resp., probability function $P$ on $\alpha$ as $P^+_\alpha$) and delay its definition till later.
$P_\bot$ is conventionally defined to be the absurd probability function which is defined to be $P_\bot(\delta)=1$ for all $\delta\in L$.

\citet{g88} proposed six rationality postulates for probabilistic belief revision. (Unless stated otherwise, it is assumed that $\alpha$ is logically satisfiable, i.e., $\vdash\lnot\alpha$ is false.)
\begin{enumerate}
\itemsep=0pt
\item $P^*_\alpha$ is a probability function
\item $P^*_\alpha(\alpha)=1$
\item If $\alpha\equiv\beta$, then $P^*_\alpha=P^*_\beta$
\item $P^*_\alpha \neq P_\bot$ iff not $\vdash\lnot\alpha$
\item If $P(\alpha)>0$, then $P^*_\alpha = P^+_\alpha$
\item If $P^*_\alpha(\beta)>0$, then $P^*_{\alpha\land\beta}=(P^*_\alpha)^+_\beta$.
\end{enumerate}

Instead of saying that the result of an operation is $P_\bot$, I simply say that the result is undefined. And by noting that the result of an operation is not a belief state if it is undefined, one can merge postulates 1 and 4.
The stochastic belief revision postulates in my notation are thus
\begin{itemize}
\itemsep=0pt
\item[]($P^*1$) $b^*_\alpha$ is a belief state iff not $\vdash\lnot\alpha$
\item[]($P^*2$) $b^*_\alpha(\alpha)=1$
\item[]($P^*3$) If $\alpha\equiv\beta$, then $b^*_\alpha=b^*_\beta$
\item[]($P^*4$) If $b(\alpha)>0$, then $b^*_\alpha = b^+_\alpha$
\item[]($P^*5$) If $b^*_\alpha(\beta)>0$, then $b^*_{\alpha\land\beta}=(b^*_\alpha)^+_\beta$.
\end{itemize}
I now test $\mathsf{OGI}$ and $\mathsf{BCI}$ against each of the five postulates.

Recall that if the denominator in the definition of $\mathsf{OGI}$ is zero, it is undefined.
To guarantee that $\mathsf{OGI}$ is defined, $\sum_{w'\in W}O(\alpha,w')b^\mathsf{GI}_\alpha(w')$ must be non-zero, that is, there must be at least one $w'\in W$ for which $O(\alpha,w')b^\mathsf{GI}_\alpha(w')>0$.
We know that when $w'\not\in||\alpha||$, $O(\alpha,w')b^\mathsf{GI}_\alpha(w') = b^\mathsf{GI}_\alpha(w') = 0$.
\begin{definition}
We say $\alpha$ is \emph{weakly observable} iff there exists a $w\in W$ such that $w\models\alpha$ and $O(\alpha,w)>0$.
We say $\alpha$ is \emph{strongly observable} iff for all $w\in W$ for which $w\models\alpha$, $O(\alpha,w)>0$.
\end{definition}

\begin{proposition}
\label{prp:revision-postulate-OGI-1}
When $*$ is $\mathsf{OGI}$, postulate ($P^*1$), in general, does not hold, but does hold if evidence $\alpha$ is strongly observable.
\end{proposition}
\begin{proof}
Firstly, observe that $b(w')=\sum_{|\mathit{Min}(\alpha,w',d)|}b(w')/|\mathit{Min}(\alpha,w',d)|$.
Therefore,
\begin{eqnarray*}
1	&=&\sum_{w'\in w}b(w')\\
	&=&\sum_{w'\in w}\sum_{|\mathit{Min}(\alpha,w',d)|}b(w')/|\mathit{Min}(\alpha,w',d)|
					\\
	&=&\sum_{w'\in w,|\mathit{Min}(\alpha,w',d)|}b(w')/|\mathit{Min}(\alpha,w',d)|\\							&=&\sum_{w'\in w,w\in W,w\in\mathit{Min}(\alpha,w',d)}b(w')/|\mathit{Min}(\alpha,w',d)|\\
	&=&\sum_{w\in W}\sum_{w'\in w,w\in\mathit{Min}(\alpha,w',d)}b(w')/|\mathit{Min}(\alpha,w',d)|\\
	&=&\sum_{w\in W}b^\mathsf{GI}_\alpha(w).
\end{eqnarray*}

Let $b=\langle0,0.1,0,0.2,0,0.3,0,0.4\rangle$ and $\alpha$ be
$(q\land r)\lor(q\land\lnot r\land s)$.
Let $O(\alpha,111)=0.9$ and $O(\alpha,w)=0$ for all $w\in W$, $w\neq111$.
(Notice that $\alpha$ is weakly observable.)
From Example~\ref{ex:2}, we know that $b^\mathsf{GI}_\alpha(111)=0$, implying that $b^\mathsf{OGI}_\alpha(111)=0$, and one can deduce that $b^\mathsf{OGI}_\alpha(w)=0$ for all $w\in W$, due to the specification of the observation model.

Now, let $\alpha$ be strongly observable: let $O(\alpha,111)=O(\alpha,110)=O(\alpha,101)=0.1$, else $O(\alpha,\cdot)=0$. Then $b^\mathsf{OGI}_\alpha = \langle0,0.7,0.3,0,0,0,0,0\rangle$.
In general, let $O(\alpha,w)>0$ for all $w\in W$ for which $w\models\alpha$. By definition of $\mathsf{GI}$, the probability mass of all non-$\alpha$-worlds is shifted to their closest $\alpha$-worlds; the total mass (of the $\alpha$-worlds) thus remains 1. Hence, $b^\mathsf{GI}_\alpha(\alpha)=1$ and there exists a $w'\models\alpha$ s.t. $b^\mathsf{GI}_\alpha(w')>0$. Now, by definition of strong observability,  $O(\alpha,w')>0$. Therefore, $O(\alpha,w')b^\mathsf{GI}_\alpha(w')>0$. And due to the normalizing effect of the denominator in the definition of $\mathsf{OGI}$, $b^\mathsf{OGI}_\alpha$ is a belief state.
\end{proof}

\begin{proposition}
\label{prp:revision-postulate-OGI-2}
When $*$ is $\mathsf{OGI}$, postulate ($P^*2$), in general, does not hold and does hold when $\alpha$ is strongly observable.
\end{proposition}
\begin{proof}
This result follows directly from an understanding of the proof of Proposition~\ref{prp:revision-postulate-OGI-1}.
\end{proof}

\begin{proposition}
\label{prp:revision-postulate-OGI-3}
When $*$ is $\mathsf{OGI}$, postulate ($P^*3$) holds.
\end{proposition}

\begin{proposition}
\label{prp:revision-postulate-OGI-4}
Let $*$ be $\mathsf{OGI}$. If $+$ is $\mathsf{OGI}$, postulate ($P^*4$) holds, otherwise it does not.
\end{proposition}

Assuming ($P^*4$) holds,  I consider whether ($P^*5$) holds only for two combinations of instantiations of $*$ and $+$.
\begin{proposition}
\label{prp:revision-postulate-OGI-5}
When $*$ is $\mathsf{OGI}$ and $+$ is $\mathsf{OGI}$, postulate ($P^*5$) does not hold.
\end{proposition}
\begin{proof}
An instance is provided where $b^\mathsf{OGI}_\alpha(\beta)>0$ and $b^\mathsf{OGI}_{\alpha\land\beta}\neq(b^\mathsf{OGI}_\alpha)^{\mathsf{OGI}}_\beta$.

Continuing with Example~\ref{ex:3}, where $b=\langle0,0.1,0,0.2,0,0.3,0,0.4\rangle$, $\alpha$ is $(q\land r)\lor(q\land\lnot r\land s)$ and $b^\mathsf{OGI}_\alpha = b^\mathsf{GI}_\alpha = \langle0,0.7,0.3,0,0,0,0,0\rangle$.
Let $\beta$ be $q\land r$, then $b^\mathsf{OGI}_\alpha(\beta) =0.7>0$.
But $b^\mathsf{OGI}_{\alpha\land\beta} = b^\mathsf{OGI}_\beta = \langle0,1,0,0,0,0,0,0\rangle$ and $(b^\mathsf{OGI}_\alpha)^\mathsf{OGI}_\beta = \langle0.3,0.7,0,0,0,0,0,0\rangle$.
\end{proof}

\begin{proposition}
\label{prp:revision-postulate-BCI-1}
When $*$ is $\mathsf{BCI}$, postulate ($P^*1$) holds.
\end{proposition}
\begin{proof}
It is known that Bayesian conditioning results in a belief state when the conditional is non-contradictory.
\end{proof}

\begin{proposition}
\label{prp:revision-postulate-BCI-2}
When $*$ is $\mathsf{BCI}$, postulate ($P^*2$) holds.
\end{proposition}
\begin{proof}
By definition of $\mathsf{BC}$, all non-$\alpha$-worlds get zero probability and the probabilities of the remaining $\alpha$-worlds are magnified to sum to 1.
\end{proof}

\begin{proposition}
\label{prp:revision-postulate-BCI-3}
When $*$ is $\mathsf{BCI}$, postulate ($P^*3$) holds.
\end{proposition}

\begin{proposition}
\label{prp:revision-postulate-BCI-4}
Let $*$ be $\mathsf{BCI}$. If $+$ is $\mathsf{BC}$ or $\mathsf{BCI}$, postulate ($P^*4$) holds, otherwise it does not.
\end{proposition}

For the proof of the next proposition, a lemma is required.
\begin{lemma}
\label{lm:1}
Let $b(\alpha)>0$.
If $b^\mathsf{BC}_\alpha(\beta)>0$, then $b(\alpha\land\beta)>0$.
\end{lemma}
\begin{proof}
Assume $b^\mathsf{BC}_\alpha(\beta)>0$.
Then there exists a $w^\beta\in W$ s.t. $w^\beta\models\beta$ and $b^\mathsf{BC}_\alpha(w^\beta)>0$.
By definition, $b^\mathsf{BC}_\alpha(w)=\frac{b(\alpha,w)}{b(\alpha)}$, implying $\frac{b(\alpha,w^\beta)}{b(\alpha)}>0$.
Hence, $b(w^\beta)>0$ and $w^\beta\models\alpha$. But if $w^\beta\models\alpha$, then  $w^\beta\models\alpha\land \beta$, and due to $b(w^\beta)>0$, $b(\alpha\land\beta)>0$.
\end{proof}

\begin{proposition}
\label{prp:revision-postulate-BCI-5}
When $*$ is $\mathsf{BCI}$ and $+$ is $\mathsf{BC}$, postulate ($P^*5$), in general, does not hold, but does hold when $b(\alpha)>0$.
\end{proposition}

Recall that a probabilistic belief change operation $\circ$ is \emph{preservative} iff for all belief states $b$ and for all propositions $\alpha$ and $\beta$, if $b(\alpha)>0$ and $b(\beta)=1$, then $b^\circ_\alpha(\beta)=1$.

\begin{proposition}
Operation $\mathsf{OGI}$ is not preservative, while  $\mathsf{BCI}$ is preservative.
\label{prp:revision-is-preservative}
\end{proposition}
\begin{proof}
$\mathsf{OGI}$:
Let the vocabulary be $\{q,r\}$ and $b=\langle0,0.5,0.5,0\rangle$. Let $\alpha$ be $q$ and $\beta$ be $q\leftrightarrow \lnot r$. Then $b(\alpha)>0$, $b(\beta)=1$ and $b^\mathsf{GI}_\alpha=\langle0.5,0.5,0,0\rangle$.
Let $O(q,w)=1$ for all $w\in W$. Then $b^\mathsf{OGI}_\alpha=\langle0.5,0.5,0,0\rangle$ and $b^\mathsf{OGI}_\alpha(\beta)=0.5$.\footnote{Here, $d$ is Hamming distance.}

$\mathsf{BCI}$:
$b^\mathsf{BCI}_\alpha(\beta) = b^\mathsf{BC}_\alpha(\beta)$.
By assuming that $b(\alpha)>0$ and $b(\beta)=1$, one is implicitly assuming that if $w\models\alpha$ s.t. $b(\alpha)>0$, then $w\models\beta$.
This in turn implies that whenever $b^\mathsf{BC}_\alpha(w)>0$, that $w\models\beta$. The latter is due to conditionalization: $\{w\in W\mid b^\mathsf{BC}_\alpha(w)>0\}$ is a subset of $\{w\in W\mid w\models\alpha, b(w)>0\}$. By ($P^*2$), $b^\mathsf{BC}_\alpha(\alpha)=1$. But due to the fact that for all $w\in W$, if $b^\mathsf{BC}_\alpha(w)>0$, then $w\models\beta$, it must then be the case that $b^\mathsf{BC}_\alpha(\beta)=1$.
\end{proof}

\subsection{Ontic and Epistemic Strength}

Suppose there is a range of degrees for information being ontic (the effect of a physical action or occurrence) or epistemic (purely informative). I shall assume that the higher the information's degree of being ontic, the lower the epistemic status of that information.
An agent has a certain sense of the degree to which a piece of received information is due to a physical action or event in the world. This sense may come about due to a combination of sensor readings and reasoning. If the agent performs an action and a change in the local environment matches the expected effect of the action, it can be quite certain that the effect is ontic information. If the agent receives the information from another agent (e.g., radio, through reading, a person speaking directly to the agent), then it should be clear to the agent that the information is epistemic and thus has a low degree of being ontic. If the agent's sensors show activity, but the agent knows that it did not presently perform an action with an effect matching its sensor readings, and if the readings do not reveal an epistemic source for the information, then the agent will have to infer from the present world conditions and the information received, or access learnt knowledge matching the present world conditions and the information received, the degree to which the information should be regarded as ontic. For instance, a person might stop talking just after you ask him/her to be quiet. Under particular conditions the person may stop talking due to your request and in other conditions he/she may have stopped talking anyway. Depending on the present world conditions, you might assign a higher (but not definitely certainty) or lower (but not definitely zero) degree of likelihood that the information (i.e., that the person stopped talking) is ontic.
Or suppose you have been wearing dark glasses for one hour. You put them on due to the sky being clear and (too) bright. When you take your glasses off, it is not as bright as you thought it would be. So, has the ambient brightness decreased due to changes in the weather, or does it only seem darker when you remove your glasses, due to some unknown physiological process? In this case, it would be convenient to consider the brightness/darkness information as being equally likely ontic and epistemic.

Recall from Definition~\ref{def:sbc-model} that $\mathit{os}(\alpha,w)$ indicates an agent's sense for the ontic  strength of $\alpha$ received in $w$.
We say that $\mathit{os}(\alpha,w)=1$ when $\alpha$ is certainly ontic in $w$. When $\alpha$ is certainly epistemic in $w$, then $\mathit{os}(\alpha,w)=0$.
In fact, let the epistemic strength of $\alpha$ in $w$ be defined as $\mathit{es}(\alpha,w):=1-\mathit{os}(\alpha,w)$.

\subsection{Combining Update and Revision}

I propose a way of trading off the probabilistic update and probabilistic revision defined earlier, using the notion of ontic strength.

The hybrid stochastic change of belief state $b$ due to new information $\alpha$ with ontic strength (denoted $b\lhd\alpha$) is defined as
\begin{definition}
\begin{align*}
b\lhd\alpha := \Big\{ &(w,p)\mid w\in W, p=\\
&\frac{1}{\gamma}\big[(1-\mathit{os}(\alpha,w))b^*_\alpha(w)+\mathit{os}(\alpha,w)b^\diamond_\alpha(w)\big]
\Big\},
\end{align*}
where $\gamma$ is a normalizing factor so that $\sum_{w\in W}b^\lhd_\alpha(w) = 1$.
\end{definition}
Due to our assumption that $\alpha$ is observed in the \textit{arrival} world, not the \textit{departure} world, $\mathit{os}(\cdot)$ is applied to the \textit{arrival} world.

Considering the rationality postulates presented so far for belief update and revision, one can naturally suggest the following postulates for their combination.
\begin{itemize}
\itemsep=0pt
\item[]($P^\lhd 1$) $b^\lhd_\alpha$ is a belief state iff not $\vdash\lnot\alpha$
\item[]($P^\lhd 2$) $b^\lhd_\alpha(\alpha)=1$
\item[]($P^\lhd 3$) If $\alpha\equiv\beta$, then $b^\lhd_\alpha=b^\lhd_\beta$
\end{itemize}
\begin{proposition}
Postulate ($P^\lhd 1$) does not hold.
\end{proposition}

\begin{proposition}
Postulate ($P^\lhd 2$) does not hold.
\end{proposition}
\begin{proof}
($P^\lhd 2$) does not hold because ($P^\diamond 2$) does not hold.
\end{proof}

\begin{proposition}
Postulate ($P^\lhd 3$) holds.
\end{proposition}
\begin{proof}
($P^\lhd 3$) is holds because ($P^\diamond 3$) and ($P^* 3$) hold.
\end{proof}

\medskip
%
%
\begin{theorem}
If: the agent model $M$ is observation-rational, $\alpha$ is trustworthy and strongly observable, there exists an event $e\in\varepsilon$ which is event-rational and $\alpha$ is an $e$-signal, then (i) $b^\lhd_\alpha$ is a belief state iff not $\vdash\lnot\alpha$ (i.e., then ($P^\lhd 1$) is true) and (ii) $b^\lhd_\alpha(\alpha)=1$ (i.e., then ($P^\lhd 2$) is true).
\label{thm:combo-postulates1and2}
\end{theorem}
\begin{proof}
Note that by Propositions~\ref{prp:update-defined}~and~\ref{prp:update-postulate2}, ($P^\diamond 1$) and ($P^\diamond 2$) hold. And recall that ($P^*1$) and ($P^*2$) are true when $\alpha$ is strongly observable (see Props.~\ref{prp:revision-postulate-OGI-1}, \ref{prp:revision-postulate-OGI-2}, \ref{prp:revision-postulate-BCI-1} and \ref{prp:revision-postulate-BCI-2}).

(i)($P^\lhd 1$) Given the antecedents of this proposition, we know by Proposition~\ref{prp:update-postulate2} that $b^\diamond_\alpha$ is defined iff not $\vdash\lnot\alpha$. And by ($P^*4$), $b^*_\alpha$ is defined iff not $\vdash\lnot\alpha$.

($\Rightarrow$) Assume $b^\lhd_\alpha$ is defined. So there exists a $w\in W$ s.t. $b^\lhd_\alpha(w)>0$, that is, $\frac{1}{\gamma}\big[(1-\mathit{os}(\alpha,w))b^*_\alpha(w)+\mathit{os}(\alpha,w)b^\diamond_\alpha(w)\big] > 0$. Thus, either $b^*_\alpha(w)>0$ (while $1-\mathit{os}(\alpha,w)>0$) or $b^\diamond_\alpha(w)>0$ (while $\mathit{os}(\alpha,w)>0$) (or both), which implies that $b^*_\alpha$ resp. $b^\diamond_\alpha$ is defined. Therefore, not $\vdash\lnot\alpha$.

($\Leftarrow$) Assume not $\vdash\lnot\alpha$.
Then $b^*_\alpha$ and $b^\diamond_\alpha$ are defined.
This implied that there exists a $w\in W$ s.t. either $b^*_\alpha(w)>0$ or $b^\diamond_\alpha(w)>0$ (or both).
Hence, $b^\lhd_\alpha(w)>0$ and due to normalization in the definition of $\lhd$, $b^\lhd_\alpha$ is defined.

(ii)($P^\lhd 2$) $b^\lhd_\alpha(\alpha)=\sum_{w\in W,w\models\alpha}b^\lhd_\alpha(w) = \sum_{w\in W,w\models\alpha}\frac{1}{\gamma}\big[(1-\mathit{os}(\alpha,w))b^*_\alpha(w)+\mathit{os}(\alpha,w)b^\diamond_\alpha(w)\big]$, where $\gamma=\sum_{w\in W}\big[(1-\mathit{os}(\alpha,w))b^*_\alpha(w)+\mathit{os}(\alpha,w)b^\diamond_\alpha(w)\big]$.
But by ($P^* 2$) and ($P^\diamond 2$), if $w\not\models\alpha$, then $b^*_\alpha(w)=0$ and $b^\diamond_\alpha(w)=0$. Hence, $\gamma=\sum_{w\in W,w\models\alpha}\big[(1-\mathit{os}(\alpha,w))b^*_\alpha(w)+\mathit{os}(\alpha,w)b^\diamond_\alpha(w)\big]$.
Therefore, $b^\lhd_\alpha(\alpha)= \sum_{w\in W,w\models\alpha}\frac{(1-\mathit{os}(\alpha,w))b^*_\alpha(w)+\mathit{os}(\alpha,w)b^\diamond_\alpha(w)}{\sum_{w\in W,w\models\alpha}\big[(1-\mathit{os}(\alpha,w))b^*_\alpha(w)+\mathit{os}(\alpha,w)b^\diamond_\alpha(w)\big]} = 1$.
\end{proof}

\medskip
Although one cannot expect $\lhd$ to be preservative, due to probabilistic update not being preservative (Prop.~\ref{prp:update-not-preservative}), one can expect $\lhd$ to have preservative-like behaviour under particular conditions:
Recall that $\alpha$ is defined to be \emph{$\beta$-trustworthy} if for all $w\in W$, if $w\not\models\beta$, then $O(\alpha, w)=0$.
\begin{proposition}
If $b^\diamond_\alpha(\beta)$ is a belief state, $b^\lhd_\alpha(\beta)$ is a belief state, $b(\beta)=1$, $\alpha$ is $\beta$-trustworthy and $*$ is $\mathsf{BCI}$, then $b^\lhd_\alpha(\beta)=1$.
\end{proposition}
\begin{proof}
By Proposition~\ref{prp:accepting-dueto-simulation}, $b^\diamond_\alpha(\beta)=1$, when $\alpha$ is $\beta$-trustworthy.
By Proposition~\ref{prp:revision-is-preservative}, $*$ is preservative when defined as $\mathsf{BCI}$.
Then, for all $w\in W$, if $w\not\models\beta$, then $b^\diamond_\alpha(\beta) = b^*_\alpha(\beta)=0$.
Hence, for all $w\in W$, if $w\not\models\beta$, $b^\lhd_\alpha(w)=0$.
Therefore, because $b^\lhd_\alpha(\beta)$ is a belief state,
\begin{eqnarray*}
b^\lhd_\alpha(\beta) &=& 1-b^\lhd_\alpha(\lnot\beta)\\
 &=& 1-\sum_{w\in w,w\models\lnot\beta}b^\lhd_\alpha(w)\\
 &=& 1-\sum_{w\in w,w\not\models\beta}b^\lhd_\alpha(w)\\
  &=& 1-0\\
  &=& 1.
\end{eqnarray*}
\end{proof}


\section{Examples and Analysis}

HSBC is now analyzed via examples. The example domain is adapted from one of the domains in the article of \citet{b98a} -- here though, worlds are associated with probabilities, not plausibility ranks. There are eight possible worlds, depending on whether a book $B$ is inside the house (if it is not in the house, then it is assumed to be on the patio, adjacent to the lawn), whether the book is dry and whether the lawn-grass $G$ is dry. There are three events: $\mathtt{rain}$ -- it rains, $\mathtt{sprnk}$ -- the sprinkler is on, and $\mathtt{null}$ -- neither of these, the null event.\footnote{I shall assume that the null event may include some unknown events (with unknown effects).} In Boutilier's example, events are deterministic; however, events in this paper are modeled to be stochastic, to better illustrate the behaviour of the framework.

To simplify calculations and to aid the reader in understanding the results, in the following examples, the agent will associate equal epistemic/ontic strength to a particular piece of information for all worlds (per example case). I shall compute the agent's new belief state for each of $\mathit{os}(\alpha,w)\in\{0,0.25,0.5,0.75,1\}$ (for all $w\in W$), for the two cases where $\alpha$ is $\lnot\mathtt{Dry}(G)$ and where $\alpha$ is $\lnot\mathtt{Dry}(G)\land \mathtt{Dry}(B)$.

Boutilier models the agent's current (initial) epistemic state with the most plausible situation (rank 0) being $(\lnot\mathtt{Inside}(B),\mathtt{Dry}(B),\mathtt{Dry}(G))$ and the next plausible situation (rank 1) being $(\mathtt{Inside}(B),\mathtt{Dry}(B),\mathtt{Dry}(G))$. I translate this as the agent having a belief state where $b(\lnot\mathtt{Inside}(B),\mathtt{Dry}(B),\mathtt{Dry}(G))=0.67$ and $b(\mathtt{Inside}(B),\mathtt{Dry}(B),\mathtt{Dry}(G))=0.33$.
Observe that in these examples, revision as $\mathsf{OGI}$ is equivalent to revision as $\mathsf{BCI}$, because $b(\lnot\mathtt{Dry}(G)) = b(\lnot\mathtt{Dry}(G)\land \mathtt{Dry}(B)) = 0$.

The HSBC model $M=\langle W,\varepsilon,T,E,O,\mathit{os}\rangle$ is now specified.

Let $w_1, \ldots, w_8$ refer to worlds
\begin{tabular}{ll}
$w_1$:&$(\mathtt{Inside}(B),\mathtt{Dry}(B),\mathtt{Dry}(G))$\\
$w_2$:&$(\mathtt{Inside}(B),\mathtt{Dry}(B),\lnot\mathtt{Dry}(G))$\\
$w_3$:&$(\mathtt{Inside}(B),\lnot\mathtt{Dry}(B),\mathtt{Dry}(G))$\\
$w_4$:&$(\mathtt{Inside}(B),\lnot\mathtt{Dry}(B),\lnot\mathtt{Dry}(G))$\\
$w_5$:&$(\lnot\mathtt{Inside}(B),\mathtt{Dry}(B),\mathtt{Dry}(G))$\\
$w_6$:&$(\lnot\mathtt{Inside}(B),\mathtt{Dry}(B),\lnot\mathtt{Dry}(G))$\\
$w_7$:&$(\lnot\mathtt{Inside}(B),\lnot\mathtt{Dry}(B),\mathtt{Dry}(G))$\\
$w_8$:&$(\lnot\mathtt{Inside}(B),\lnot\mathtt{Dry}(B),\lnot\mathtt{Dry}(G))$
\end{tabular}

\medskip
The events are $\varepsilon=\{\mathtt{rain},\mathtt{sprnk},\mathtt{null}\}$.

The following probabilities are debatable; they should not be taken too seriously but serve to illustrate the framework.

\medskip
{\hspace{-0.5cm}
\begin{tabular}{l|l}
$T(w_1,\mathtt{null},w_1)=0.75$ & $T(w_5,\mathtt{null},w_5)=1$\\
$T(w_1,\mathtt{null},w_2)=0.1$ & $T(w_5,\mathtt{null},w_6)=0$\\
$T(w_1,\mathtt{null},w_3)=0.1$ & $T(w_5,\mathtt{null},w_7)=0$\\
$T(w_1,\mathtt{null},w_4)=0.05$ & $T(w_5,\mathtt{null},w_8)=0$\\
\hline\\
$T(w_1,\mathtt{rain},w_1)=0$ & $T(w_5,\mathtt{rain},w_5)=0$\\
$T(w_1,\mathtt{rain},w_2)=0.75$ & $T(w_5,\mathtt{rain},w_6)=0.05$\\
$T(w_1,\mathtt{rain},w_3)=0$ & $T(w_5,\mathtt{rain},w_7)=0.05$\\
$T(w_1,\mathtt{rain},w_4)=0.25$ & $T(w_5,\mathtt{rain},w_8)=0.9$\\
\hline\\
$T(w_1,\mathtt{sprnk},w_1)=0$ & $T(w_5,\mathtt{sprnk},w_5)=0$\\
$T(w_1,\mathtt{sprnk},w_2)=0.8$ & $T(w_5,\mathtt{sprnk},w_6)=0.8$\\
$T(w_1,\mathtt{sprnk},w_3)=0$ & $T(w_5,\mathtt{sprnk},w_7)=0.05$\\
$T(w_1,\mathtt{sprnk},w_4)=0.2$ & $T(w_5,\mathtt{sprnk},w_8)=0.15$
\end{tabular}
}

\medskip
\begin{tabular}{ll}
$E(\mathtt{null}, w_1)=0.06$ & $E(\mathtt{null}, w_5)=0.15$\\
$E(\mathtt{rain}, w_1)=0.31$ & $E(\mathtt{rain}, w_5)=0.7$\\
$E(\mathtt{sprnk}, w_1)=0.63$ & $E(\mathtt{sprnk}, w_5)=0.15$
\end{tabular}

{\hspace{-0.7cm}
\begin{tabular}{ll}
$O(\lnot\mathtt{Dry}(G), w_1)=0.05$ &$O(\lnot\mathtt{Dry}(G)\land\mathtt{Dry}(B), w_1)=0.5$\\
$O(\lnot\mathtt{Dry}(G), w_2)=0.95$ &$O(\lnot\mathtt{Dry}(G)\land\mathtt{Dry}(B), w_2)=0.8$\\
$O(\lnot\mathtt{Dry}(G), w_3)=0.05$ &$O(\lnot\mathtt{Dry}(G)\land\mathtt{Dry}(B), w_3)=0.1$\\
$O(\lnot\mathtt{Dry}(G), w_4)=0.95$ &$O(\lnot\mathtt{Dry}(G)\land\mathtt{Dry}(B), w_4)=0.05$\\ 
$O(\lnot\mathtt{Dry}(G), w_5)=0.05$ &$O(\lnot\mathtt{Dry}(G)\land\mathtt{Dry}(B), w_5)=0.6$\\
$O(\lnot\mathtt{Dry}(G), w_6)=0.95$ &$O(\lnot\mathtt{Dry}(G)\land\mathtt{Dry}(B), w_6)=0.98$\\
$O(\lnot\mathtt{Dry}(G), w_7)=0.05$ &$O(\lnot\mathtt{Dry}(G)\land\mathtt{Dry}(B), w_7)=0.2$\\
$O(\lnot\mathtt{Dry}(G), w_8)=0.95$ &$O(\lnot\mathtt{Dry}(G)\land\mathtt{Dry}(B),w_8)=0.15$
\end{tabular}
}


\medskip
Recall that the current belief state is $b=\langle 0.33,0,0,0,0.67,0,0,0 \rangle$.
The following is a list of resulting belief states $b'=b\lhd\lnot\mathtt{Dry}(G)$ for the specified ontic strengths.

\medskip
\begin{tabular}{rc}
$\mathit{os}(\cdot)$ & $b\lhd\lnot\mathtt{Dry}(G)$\\
\hline\\
0.00&$\langle 0.00,0.33,0.00,0.00,0.00,0.67,0.00,0.00 \rangle$\\
0.25&$\langle 0.00,0.32,0.00,0.02,0.00,0.53,0.00,0.13 \rangle$\\
0.50&$\langle 0.00,0.31,0.00,0.04,0.00,0.40,0.00,0.25 \rangle$\\
0.75&$\langle 0.00,0.30,0.00,0.06,0.00,0.26,0.00,0.38 \rangle$\\
1.00&$\langle 0.00,0.28,0.00,0.08,0.01,0.12,0.00,0.51 \rangle$
\end{tabular}

\medskip
Several behaviours can be noted:
When the observation is completely epistemic, the probabilities of the two believed worlds are each shifted to their closest $\lnot\mathtt{Dry}(G)$-worlds.
The more the agent considers the information to be ontic, the more its beliefs are spread out due to the nondeterminism of the events.
Whether the observation is considered ontic or epistemic, the agent has a relatively strong belief (between $28\%$ and $33\%$) that the book is inside and dry.
However, in cases where the book is outside, there is a considerable shift in probability from the book being dry ($w_6$) to it being wet ($w_8$), as the agent moves towards an ontic mindset.
One could perhaps argue that in an ontic mindset, the agent has access to event/transition information so as to reason about the causes of the book getting wet: 
it believes that there is a moderate to high likelihood that the book will get wet if it is on the patio, due to the sprinkler coming on or it starting to rain (explaining the wet-grass evidence).


The following is a list of resulting belief states $b'=b\lhd\lnot\mathtt{Dry}(G)\land\mathtt{Dry}(B)$ for the specified epistemic strengths.

\medskip
\begin{tabular}{rc}
$\mathit{os}(\cdot)$ & $b\lhd\lnot\mathtt{Dry}(G)\land\mathtt{Dry}(B)$\\
\hline\\
0.00&$\langle 0.00,0.29,0.00,0.00,0.00,0.71,0.00,0.00 \rangle$\\
0.25&$\langle 0.00,0.33,0.00,0.00,0.03,0.59,0.00,0.04 \rangle$\\
0.50&$\langle 0.01,0.37,0.00,0.00,0.07,0.47,0.01,0.07 \rangle$\\
0.75&$\langle 0.01,0.41,0.00,0.01,0.10,0.35,0.01,0.11 \rangle$\\
1.00&$\langle 0.02,0.45,0.00,0.01,0.14,0.23,0.01,0.15 \rangle$
\end{tabular}

\medskip
When the agent considers the observation completely epistemically, its beliefs change very similarly to when it was only told that the grass is wet; the agent already believed that the book was dry.
However, the extra information has a significant impact on how the agent's beliefs change when the observation is considered ontically.
The agent now believes much less that the book is outside and wet and the grass is wet, and with $78\%$ (as opposed to $40\%$ with the first observation) that the book is \textit{dry} and the grass is wet (independent of where the book is located).
The reason is that when the received information includes a dry book, transitions are focused on going to dry-book worlds.


\section{Conclusion}

In this paper I suggested a method to arrive at a new (probabilistic) belief state when the agent has mixed feelings about whether to revise or update its beliefs, given a new piece of information.
Much attention was given to the design and analysis of the separate update and revision operations.
The postulates and finally Theorem~\ref{thm:combo-postulates1and2} add weight to my argument that the hybrid stochastic belief change (HSBC) operation is rational when the agent has a rational frame of mind.

Looking at the examples above, the way in which probabilities shift among the possible worlds, given the different ontic/epistemic strengths, seems justifiable. However, more analysis is required here, especially when considering more complicated specification patterns of the ontic/epistemic strengths.

Determining $\mathit{os}(\alpha,w)$ for every foreseen $\alpha$ in every possible world $w$ will be challenging for a designer. Some deep questions are: Should the designer/agent provide the strengths (via stored values or programmed reasoning), or do these strengths come to the agent \textit{attached} to the new information?
What is the reasoning process we go through to determine whether information is epistemic or ontic, if at all?
In general, how does an agent know when information is epistemic (requiring revision) or ontic (requiring update)?

%
%
%

\bibliographystyle{aaai}
\bibliography{references}

\end{document}